\documentclass{article}
\usepackage[final,sglblindworkshop]{neurips_2026}
\workshoptitle{ML$\times$OR: Mathematical Foundations and Operational Integration of Machine Learning for Uncertainty-Aware Decision-Making}

\usepackage[utf8]{inputenc}
\usepackage[T1]{fontenc}
\usepackage[hyphens]{url}
\usepackage{booktabs}
\usepackage{amsmath,amssymb,amsthm,mathtools}
\usepackage{graphicx}
\usepackage{microtype}
\usepackage{color}

\theoremstyle{plain}
\newtheorem{proposition}{Proposition}

\newcommand{\R}{\mathbb{R}}
\newcommand{\E}{\mathbb{E}}
\newcommand{\Prob}{\mathbb{P}}
\newcommand{\T}{\mathcal{T}}
\newcommand{\X}{\mathcal{X}}
\DeclareMathOperator*{\argmax}{arg\,max}
\DeclareMathOperator*{\argmin}{arg\,min}

\newcommand{\Dlo}{\underline{D}_\tau}
\newcommand{\Dup}{\overline{D}_\tau}

\title{Differentiating Through Dual Prices: End-to-End Policy Learning Under Capacity Constraints}

\author{%
  \textbf{Mohammadsaeed Haghi}\thanks{Equal contribution.}
  \quad \textbf{Mahdi Salmani}\footnotemark[1]
  \quad \textbf{Nima Kelidari}\footnotemark[1]\\
  \textbf{Vishal Gupta}
  \quad \textbf{Phebe Vayanos}\\
  University of Southern California\\
  \texttt{\{haghim, salmanis, kelidari, guptavis, phebe.vayanos\}@usc.edu}
}
\newtheorem{remark}{Remark}          %
\usepackage{algorithm}               %
\usepackage{algpseudocode}           %
\DeclareMathOperator{\softmax}{softmax}   %

\begin{document}

\maketitle

\begin{abstract}
Many social services assign scarce resources, such as housing assistance or hospital interventions, to people who arrive one at a time. Each arrival must receive a decision immediately, and the long-run usage of every resource must stay within its capacity. We study how to learn such a policy from logged observational data. Recent work proposes a decision-blind pipeline that fits one outcome model per treatment by regression, prices each capacitated resource from those models, and assigns each arrival the treatment with the largest predicted outcome minus price \citep{tang2024learning}. We instead train the outcome models end-to-end, differentiating an off-policy estimate of the deployed policy's value through the dual prices themselves. We study two formulations: an exact nonconvex one, and a convex relaxation whose optimum always satisfies the capacity constraints in expectation and
whose suboptimality can be explicitly bounded. Every method is evaluated in a queueing simulation
to capture the online decision structure.
Across six datasets, our two end-to-end variants
lead the delay-adjusted policy value wherever capacity is scarce, once a waiting period carries even a small cost.
In particular, decision-blind regression predicts ground truth more accurately, but when resource capacities are stringent these baselines frequently violate them and incur much longer queueing delays.
Thus, end-to-end training is best suited to settings where resources are genuinely scarce and feasibility matters.
\end{abstract}

\section{Introduction}
\label{sec:intro}

Many social services must assign scarce resources to people who arrive one at a time: housing assistance \citep{tang2024learning}, refugee resettlement \citep{bansak2024paulson}, donor organs \citep{bertsimas2013fairness}, and clinical triage \citep{argon2009priority}. Decisions are online. When a person arrives, the planner observes their covariates, must choose a treatment immediately, and cannot revisit the choice. Capacity couples these decisions. Each resource can absorb only a fraction of the population in the long run, and a unit given to one person is unavailable until replenished. We study how to learn such a policy from logged observational data.

\citet{tang2024learning} give a predict-then-optimize (PtO) approach with convergence guarantees, but these assume an unbiased outcome estimator per treatment, whereas allocation depends only on which treatment ranks highest after pricing. A prediction error that flips the top treatment barely moves the regression loss, so we differentiate through the pricing step and train for decision quality directly. Such rules are cheap to evaluate and keep long-run usage within capacity, and are standard in revenue management \citep{talluri1998analysis}. However, fitting per-treatment regressions ignores capacities and only indirectly reflects policy value. Errors near a price boundary reassign scarce resources and create deployment queues, an effect we quantify in Section~\ref{sec:experiments}.

Two separate research streams each partially address these weaknesses. Off-policy policy learning \citep{swaminathan2015batch,kitagawa2018should,athey2021policy} trains a policy directly on an estimate of its value, but its policy classes carry no notion of long-run capacity, so feasibility must be restored after the fact. On the other hand, decision-focused learning \citep{elmachtoub2022smart,donti2017task,wilder2019melding,amos2017optnet,mandi2024decision} trains predictors against a downstream optimization problem, but generally assumes fully observed outcomes and per-context feasible regions.  In our logged setting we observe outcomes only for the assigned treatment \citep{hu2025contextuallinearoptimizationpartial}, and capacities constrain decisions \emph{across} contexts, not within one as in standard DFL.

Finally, we distinguish our work from \citet{rodriguezdiaz2025dual}, who also use dual prices in learning but with a computational motivation. Given optimal prices, their policy decouples into a fast, often closed-form backpropagation step, so they alternate gradient steps on the learner and on the prices. At deployment they call a solver on a batch, which ensures feasibility. Our setting is online, a single arrival not a batch, so no such solver is available, and non-optimal prices can then degrade performance. Differentiating through the prices, though more expensive, yields more reliably feasible and often better delay-adjusted policies.

\section{Problem Formulation and Evaluation}
\label{sec:setting}
Individual $i$ has covariates $X_i\in\X\subseteq\R^d$ drawn i.i.d.; the treatment set is $\T=\{0,1,\dots,K\}$ with $t=0$ the no-treatment option; potential outcomes are $\{Y_i^t\}_{t\in\T}$ and only $Y_i=Y_i^{T_i}$ is logged. The outcome surface $m_t(x):=\E[Y^t\mid X=x]$ is estimated from a log $\{(X_i,T_i,Y_i,e_{T_i}(X_i))\}_{i=1}^N$ with propensities $e_t(x)$ (estimated by $\widehat e_t$ when unknown), under strong ignorability and overlap; $m_t$ is the true surface and $m_t(x;\theta)$ a fitted model with parameters $\theta$.

Decisions are online (Section~\ref{sec:intro}): the planner sees each arrival and commits to a treatment at once. Each scarce treatment $t$ has a long-run capacity $b_t\in[0,1]$, the fraction of the population it can absorb per unit time ($b_0=1$). Under strict capacity an averaged policy need not be implementable, but online the consequence is not infeasibility: surplus arrivals queue, and that delay is the cost. A policy $\pi:\X\to\Delta(\T)$ may thus be infeasible in the averaged sense $\E_X[\pi_t(X)]>b_t$.

The quantity we care about is a policy's \emph{value} $V(\pi):=\E_X[\sum_{t\in\T}\pi_t(X)\,m_t(X)]$, its expected outcome under the true surface. The planner cannot evaluate $V$ directly, so the standard route is the averaged approximation: maximize value subject to average capacity, $\max_{\pi}\E_X[\sum_t\pi_t(X)\,m_t(X;\theta)]$ s.t. $\E_X[\pi_t(X)]\le b_t$ and $\sum_t\pi_t(x)=1,\ \pi_t(x)\ge0$, with the fitted model in place of $m_t$. This is not the deployment objective, however: what a planner pays online is delay, so we evaluate against the delay-adjusted policy value below.

The averaged problem has the Lagrangian dual
\begin{equation}
\label{eq:dual}
\nu^\star=\min_{\mu\ge 0}\ \E_X\Big[\max_{t\in\T}\big(m_t(X;\theta)-\mu_t\big)\Big]+\textstyle\sum_{t\in\T}\mu_t b_t,
\end{equation}
attained at prices $\mu^\star$. These prices induce the hard price-adjusted rule $\pi(x;\theta,\mu^\star)\in\argmax_{s}(m_s(x;\theta)-\mu_s^\star)$, which assigns each arrival its top-scoring treatment online; the predict-then-optimize (PtO) baseline of \citet{tang2024learning} fits $m_t(\cdot\,;\theta)$ by per-treatment regression, solves the sample analogue of \eqref{eq:dual} for $\widehat\mu$, and deploys it.

Since $V$ is not computable from a finite log, we estimate it off-policy by inverse-propensity weighting, $\widehat V_{\mathrm{IPW}}(\pi):=\tfrac1N\sum_{i=1}^N \pi_{T_i}(X_i)\,Y_i/\widehat e_{T_i}(X_i)$, for any policy feasible or not (its value on a held-out split is the held-out value). Value alone is the wrong deployment target, since it ignores the queueing an over-cap policy causes. To score deployment we run the policy through a queueing simulation \citep{tang2024learning}: an assignment to an exhausted treatment joins a FIFO queue, units replenish as a Poisson process at rate $b_t\lambda$, and the run returns the unserved fraction $u$ and mean wait $W$. Our primary metric is the \emph{delay-adjusted policy value},
\begin{equation}
\label{eq:dapv}
\mathrm{DAPV}(\kappa)=(1-u)\,\widehat V_{\mathrm{IPW}}+u\,V_0-\kappa\,W,
\end{equation}
with $V_0$ the no-treatment value credited to unserved arrivals and $\kappa\ge0$ the price of one waiting period. This delay-adjusted version of $\widehat V_{\mathrm{IPW}}$, not raw value, is what the experiments report: it charges infeasibility through $-\kappa W$.

\section{Model Setup: Differentiating Through the Prices}
\label{sec:method}
To train the outcome models end-to-end, we replace the hard argmax rule with the differentiable smoothed policy:
\begin{equation}
\label{eq:softmax-policy}
\tilde\pi_t(x;\theta,\mu)=\frac{\exp\bigl((m_t(x;\theta)-\mu_t)/\tau\bigr)}
{\sum_{s\in\T}\exp\bigl((m_s(x;\theta)-\mu_s)/\tau\bigr)},
\end{equation}
with temperature $\tau>0$, which recovers the argmax as $\tau\downarrow0$; write $\tilde\pi(x;\theta,\mu):=(\tilde\pi_t(x;\theta,\mu))_{t\in\T}$ for the resulting policy.

Under this policy, the sample version of the dual~\eqref{eq:dual} averages the price-adjusted score under $\tilde\pi$ rather than taking its maximum, which defines the inner objective $\Dlo$. Letting $M_\theta:=(m_t(X_i;\theta))_{i,t}$, we learn $\theta$ and the induced prices $\mu_\theta$ jointly through the bilevel program
\begin{equation}
\label{eq:bilevel}
\begin{aligned}
&\max_\theta\;\widehat V_{\mathrm{IPW}}\!\big(\tilde\pi(\cdot\,;\theta,\mu_\theta)\big)
\quad\text{s.t.}\quad
\mu_\theta=\argmin_{\mu\ge 0} \Dlo(\mu;M_\theta),\\[2pt]
&\qquad \Dlo(\mu;M):=\tfrac{1}{N}\textstyle\sum_{i,t} \tilde\pi_t(X_i;\theta,\mu)\bigl(M_{i,t}-\mu_t\bigr)+\sum_t b_t\mu_t,
\end{aligned}
\end{equation}
where the inner problem optimizes $\mu_\theta$ for the current model, and the outer maximizes IPW value.

The outcome model is a small shared-trunk MLP. We call this family dual-price learning (DPL) and solve \eqref{eq:bilevel} three ways (training loops in Appendix~\ref{app:algorithms}). \textbf{DPL-nc} exactly differentiates through the KKT conditions of the nonconvex inner argmin (Appendix~\ref{app:ift}). \textbf{DPL-alt} mirrors \citet{rodriguezdiaz2025dual}, alternating gradient steps on $\theta$ and $\mu$ without implicit differentiation. \textbf{DPL-cvx} replaces $\Dlo$ with its convex log-sum-exp upper bound, the entropic smoothing of the same sample dual,
\begin{equation}
\label{eq:dup}
\Dup(\mu;M):=\frac{1}{N}\sum_{i}\tau\log\!\sum_{t}\exp\!\bigl((M_{i,t}-\mu_t)/\tau\bigr)+\textstyle\sum_t b_t\mu_t.
\end{equation}

Two results justify this surrogate (proofs in Appendices~\ref{app:bracket} and~\ref{app:feasibility}):
\begin{proposition}[Dual bracket]
\label{prop:bracket}
For every $\mu\ge0$, $M$, $\tau>0$:\; $\Dlo(\mu;M)\le \Dup(\mu;M)\le \Dlo(\mu;M)+\tau\log|\T|$.
\end{proposition}
The hard dual objective of \eqref{eq:dual} lies between $\Dlo$ and $\Dup$, so the bracket has width at most $\tau\log|\T|$. Shrinking $\tau$ closes it but costs conditioning: the inner Hessian $\nabla^2_\mu\Dup=\tfrac1\tau\E[\operatorname{diag}(\tilde\pi)-\tilde\pi\tilde\pi^{\!\top}]$ scales as $1/\tau$ on priced coordinates, so as $\tau\downarrow0$ the solve becomes ill-conditioned and the policy gradient vanishes off the decision boundary. We therefore fix $\tau$ a priori ($\tau=0.03$ nonconvex, $0.1$ convex).

\begin{proposition}[Exact feasibility in expectation]
\label{prop:feas}
Let $\mu^\star\in\argmin_{\mu\ge 0}\Dup(\mu;M)$ and $\tilde\pi^\star$ the policy \eqref{eq:softmax-policy} at $(M,\mu^\star)$. Then $\frac{1}{N}\sum_{i}\tilde\pi^\star_t(X_i)\le b_t$ for every $t$, with equality whenever $\mu^\star_t>0$.
\end{proposition}
Feasibility is thus a structural property of what \textbf{DPL-cvx} learns: for any $\theta$, however inaccurate the models, the priced policy respects every capacity in expectation. In a decision-blind pipeline, feasibility is disconnected from training. Here the prices that enforce feasibility are the prices we differentiate through, so the guarantee and the training signal come from one object.

\section{Experiments}
\label{sec:experiments}

We test whether differentiating through capacity prices improves feasibility without sacrificing policy value. We compare our methods, DPL-nc and DPL-cvx, with matched PtO-mlp, following \citet{tang2024learning}, and DPL-alt, an alternating-price baseline based on \citet{rodriguezdiaz2025dual} that does not differentiate through the price solution. PtO-mlp uses the same network width, learning rate, and training budget as our score model. This matching makes the training objective, rather than model capacity or optimisation effort, the main difference between the methods.

We organize the six datasets into two groups: synthetic or semi-synthetic data (ten-treatment synthetic, dense-margin synthetic, and Adult semi-synthetic) and real-world data (ACTG~175, Diabetes~130-US, and Criteo Uplift). ACTG and Diabetes play complementary roles. ACTG provides a randomised comparison with known propensities, while Diabetes tests the method at a much larger observational scale where ignorability is required.

Each policy is evaluated over $1{,}000$ sequential arrivals. Results average $10$ seeds using fixed hyperparameters. Appendices~\ref{app:protocol} and~\ref{app:suite} give the complete protocol and dataset details.

Table~\ref{tab:results} shows the main pattern. DPL-nc and DPL-cvx keep the busiest scarce treatment at or below capacity on all six datasets, whereas PtO-mlp violates capacity on four and reaches $116\%$ on the ten-treatment synthetic. On synthetic and semi-synthetic data, DPL-nc improves value on both fully synthetic designs---by up to $72\%$---but loses on Adult semi-synthetic. On real-world data, it matches or improves PtO-mlp in practical terms across all three datasets, with the clearest gain on Diabetes.

\begin{table}[t]
\centering\small
\caption{Deployment results at full sample size ($10$ seeds). Usage above $100\%$ violates capacity; the final column gives DPL-nc value relative to PtO-mlp (mean $\pm$ standard error).}
\label{tab:results}
\setlength{\tabcolsep}{6pt}
\renewcommand{\arraystretch}{1.15}
\begin{tabular}{@{}l cccc @{\hspace{1.6em}} cccc @{\hspace{1.6em}} c@{}}
\toprule
 & \multicolumn{4}{c}{Average horizon usage (\% of cap)} & \multicolumn{4}{c}{Mean wait} & \\
\cmidrule(lr){2-5}\cmidrule(lr){6-9}
Dataset & \textsc{nc} & \textsc{cvx} & \textsc{alt} & \textsc{mlp} & \textsc{nc} & \textsc{cvx} & \textsc{alt} & \textsc{mlp} & DPL-nc vs.\ \textsc{mlp} \\
\midrule
ten-treatment synth. & $99$ & $97$ & $96$ & $116$ & $20.0$ & $20.7$ & $21.0$ & $39.5$ & $+72\pm13\%$ \\
dense-margin synth. & $60$ & $61$ & $62$ & $105$ & $1.3$ & $1.2$ & $0.7$ & $13.5$ & $+11\pm5\%$ \\
Adult semi-synthetic & $99$ & $99$ & $98$ & $105$ & $21.3$ & $15.8$ & $20.1$ & $28.9$ & $-15\pm2\%$ \\
Diabetes 130-US & $95$ & $95$ & $86$ & $102$ & $5.5$ & $5.3$ & $2.4$ & $15.6$ & $+5\pm1\%$ \\
ACTG 175 & $98$ & $100$ & $103$ & $99$ & $9.1$ & $15.0$ & $12.7$ & $13.8$ & $-1\pm2\%$ \\
Criteo Uplift & $92$ & $91$ & $91$ & $90$ & $3.5$ & $2.6$ & $3.9$ & $3.5$ & $+16\pm2\%$ \\
\bottomrule
\end{tabular}
\end{table}

DPL-alt can be competitive in settings with few treatments, but its performance is less consistent: it violates capacity on ACTG and loses substantial value on the ten-treatment synthetic. Alternating price updates are therefore not a reliable substitute for differentiating through the optimised prices across these settings; Appendix~\ref{app:alt} reports the detailed comparison.

The delay-cost sweep in Appendix~\ref{app:sweep} shows that DPL-nc has higher DAPV than PtO-mlp at every tested cost on the ten-treatment synthetic, dense-margin, and Diabetes datasets. The sweep remains informative even at $\kappa=0$, because unserved arrivals fall back to the no-treatment value. An over-cap policy can therefore lose deployed policy value before any explicit waiting penalty is applied.

DPL-nc overtakes PtO-mlp on ACTG and Adult semi-synthetic once waiting has a modest cost, while the methods remain close on slack-capacity Criteo. Appendix~\ref{app:mechanism} examines the allocation-boundary mechanism and the nonbinding control.

Taken together, the results show that feasibility and policy value need not be traded off. DPL controls average demand while preserving policy value in most settings, whereas exceeding capacity does not reliably improve deployed value and often produces longer waits. The benefit of differentiating through prices is clearest when capacity changes deployment and overload creates operational delay.

\section{Conclusion}
\label{sec:discussion}

Differentiating through capacity prices aligns training with deployment under scarce capacity. Across both dataset groups, DPL-nc and DPL-cvx control average usage more reliably than PtO-mlp, while DPL-nc generally preserves policy value. The gain is clearest when prediction errors change who receives a scarce treatment and overload creates delay; DPL-cvx additionally provides expected feasibility by construction.

The results also define the limits of the claim. DPL-nc loses raw value on Adult semi-synthetic, can leave capacity unused, and does not dominate DPL-alt in every setting with few treatments. Moreover, the real-world estimates rely on ignorability and IPW, while the simulation measures average rather than pathwise feasibility. One promising direction for extending this research is to use a doubly robust objective. This could reduce reliance on estimated propensities while preserving the capacity-aware training and deployment mechanism.

\vspace{1em}

\bibliographystyle{plainnat}
\bibliography{refs}

\appendix

\section{Experimental protocol}
\label{app:protocol}
We evaluate two proposed methods, \textbf{DPL-nc} and \textbf{DPL-cvx}, against two comparison pipelines. \textbf{PtO-mlp} is the decision-blind baseline: it fits one outcome regression per treatment and computes deployment prices only after fitting. Its network width, learning rate, and training budget match those of the DPL models, reducing differences caused by model size or training effort. \textbf{DPL-alt} is an alternating-price baseline based on \citet{rodriguezdiaz2025dual}. It refreshes the prices between blocks of model updates but treats those prices as fixed during each block, so no gradient passes through the price solution. DPL-alt is not one of our proposed methods.

For each seed, we evaluate the learned policy in a queueing simulation with $1{,}000$ sequential arrivals. Each scarce treatment replenishes at its capacity rate, and an arrival assigned to an exhausted treatment enters a first-in, first-out queue. Usage is the average assignment rate of the busiest scarce treatment divided by that treatment's capacity; a value above $100\%$ therefore means that average demand exceeds capacity. Mean wait is measured in mean interarrival periods. Policy value is computed from ground-truth potential outcomes when available and by held-out IPW otherwise. Equation~\eqref{eq:dapv} combines service-adjusted value and waiting time into DAPV.

All results average $10$ seeds at each dataset's full sample size. Across seeds, standard errors are at most five percentage points for usage and five periods for mean wait. Hyperparameters are selected on the first synthetic configuration and then held fixed across the remaining datasets. The same simulation horizon and evaluation rules are used for every method.

\section{Evaluation suite}
\label{app:suite}
\begin{table}[h]
\centering\small
\caption{Evaluation datasets. Capacities apply to scarce treatments; results average $10$ seeds.}
\label{tab:suite}
\setlength{\tabcolsep}{4pt}
\begin{tabular*}{\textwidth}{@{\extracolsep{\fill}}lrrrl@{}}
\toprule
Dataset & $N$ & Treatments & Cap & Deployed value scored by \\
\midrule
ten-treatment synthetic & 4{,}000 & 10 & $0.10$ & ground truth \\
dense-margin synthetic & 32{,}000 & 2 & $0.25$ & ground truth \\
Adult semi-synthetic \citep{kohavi1996} & 16{,}000 & 8 & $0.08$ & ground truth \\
ACTG 175 \citep{hammer1996} & 1{,}497 & 4 & $0.30$ & held-out IPW \\
Diabetes 130-US \citep{strack2014} & 48{,}981 & 3 & $0.25$ & held-out IPW \\
Criteo Uplift \citep{diemert2018criteo} & 32{,}000 & 2 & $0.50$ & held-out IPW \\
\bottomrule
\end{tabular*}
\end{table}

The suite separates settings in which capacity affects deployment from a nonbinding control. The ten-treatment synthetic concentrates value in a small high-value region when each scarce treatment can serve only $10\%$ of arrivals. The dense-margin synthetic concentrates observations near the assignment cutoff, where small score errors can change who receives treatment. Adult semi-synthetic uses real census covariates with constructed potential outcomes and seven capacity-constrained treatments in addition to no treatment. ACTG~175 is a randomised trial with known propensities, Diabetes~130-US provides the largest observational cohort, and Criteo uses a $50\%$ cap that remains slack on average.

The two synthetic datasets and Adult semi-synthetic are evaluated against ground-truth potential outcomes. ACTG, Diabetes, and Criteo are evaluated by held-out IPW. This distinction changes only how deployed value is estimated; usage and waiting time are measured by the same queueing simulation for all six datasets.

\section{Mechanism and nonbinding control}
\label{app:mechanism}
\paragraph{Dense-margin mechanism.}
In the dense-margin design, $30\%$ of the population benefits from treatment, but capacity permits only $25\%$ of arrivals to be treated. Because observations concentrate near this cutoff, small prediction errors can overload the resource. PtO-mlp achieves a value of $0.084$, reaches $105\%$ of capacity, and waits $13.5$ periods on average. DPL-nc achieves a higher value of $0.093$, uses $60\%$ of capacity, and waits $1.3$ periods; DPL-cvx uses $61\%$ of capacity and waits $1.2$ periods. Price-aware training therefore avoids overload and improves value in this design, but it does so conservatively and leaves part of the available capacity unused. The result supports better selection under capacity, not exact saturation of the resource.

\paragraph{Criteo control.}
On Criteo, the busiest-treatment usage of every method lies between $90\%$ and $92\%$ of capacity, and mean waits range from $2.6$ to $3.9$ periods. DPL-nc's value is $16\%$ higher than PtO-mlp in relative terms, but the absolute difference is only $0.006$ because the values are close to zero. Thus, the methods remain close in practical terms when capacity is slack. The nonzero waits reflect stochastic arrivals and do not indicate persistent average overload.

\section{Alternating-price baseline}
\label{app:alt}
DPL-alt isolates whether alternating model and price updates can replace differentiation through the optimised prices. It can be competitive: on Diabetes, DPL-alt achieves a value of $0.991$, compared with $0.989$ for DPL-nc, and reduces mean wait from $5.5$ to $2.4$ periods. Its performance is less consistent elsewhere. On the ten-treatment synthetic its value falls to $9.9$, compared with $39.0$ for DPL-nc, and on ACTG its average busiest-treatment usage reaches $103\%$ of capacity. These results do not establish uniform dominance by either training rule. They show that the alternating shortcut is not a reliable substitute across the range of treatment counts and capacity conditions considered here.

\section{Delay-cost sweep}
\label{app:sweep}
Rather than choose one application-specific cost of waiting, we sweep $\kappa$, expressed as a fraction of each dataset's mean treatment value. At $\kappa=0$, DAPV is not necessarily raw policy value: it still assigns the no-treatment value to arrivals left unserved at the end of the horizon. An over-cap policy can therefore lose deployment value even before an explicit waiting penalty is applied.
\begin{figure}[h]
\centering
\includegraphics[width=\textwidth]{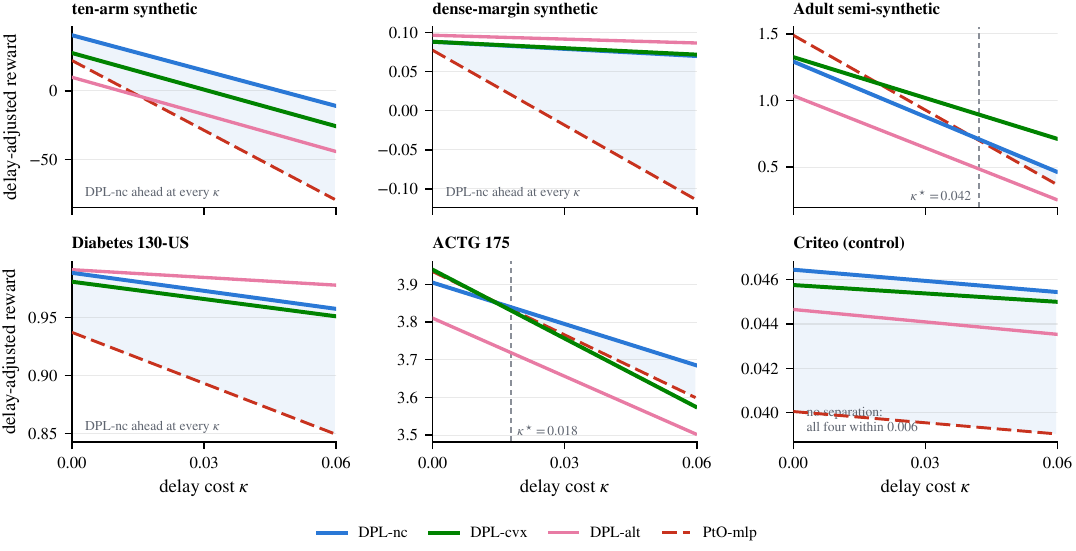}
\caption{Delay-adjusted policy value $\mathrm{DAPV}(\kappa)$ of \eqref{eq:dapv} against the delay cost $\kappa$, one panel per dataset, each in its own outcome units. Where the DPL-nc and PtO-mlp curves cross (ACTG~175 and Adult semi-synthetic), a vertical rule marks $\kappa^\star$, the delay cost at which DPL-nc overtakes PtO-mlp.}
\label{fig:index}
\end{figure}

Against PtO-mlp, DPL-nc has higher DAPV at every tested $\kappa$ on the ten-treatment synthetic, dense-margin, and Diabetes datasets. On ACTG and Adult semi-synthetic, PtO-mlp begins ahead, but DPL-nc overtakes it at $\kappa^\star=0.018$ and $0.042$, respectively, corresponding under our normalisation to $1.8\%$ and $4.2\%$ of mean treatment value per waiting period. On Criteo, no crossing occurs because capacity remains slack on average and the absolute differences are small. These statements compare DPL-nc specifically with PtO-mlp; another method may be higher in an individual panel.

\section{Proof of the Dual Bracket}
\label{app:bracket}
Throughout the appendix, $\pi_t(X_i;\theta,\mu)$ denotes the softmax weight \eqref{eq:softmax-policy} on treatment $t$ at sample $i$, and $\pi(X_i;\theta,\mu)=(\pi_t(X_i;\theta,\mu))_{t\in\T}$ the corresponding distribution.
\begin{proof}[Proof of Proposition~\ref{prop:bracket}]
Fix $i$ and write $a_{t,i}=(M_{i,t}-\mu_t)/\tau$, so $\pi_t(X_i;\theta,\mu)=e^{a_{t,i}}/\sum_s e^{a_{s,i}}$, and let $H(\pi(X_i;\theta,\mu))=-\sum_t \pi_t(X_i;\theta,\mu)\log\pi_t(X_i;\theta,\mu)\ge 0$ be its Shannon entropy. Define
\begin{gather*}
\phi_i(\mu)=\sum_t \pi_t(X_i;\theta,\mu)\,(M_{i,t}-\mu_t)=\tau\sum_t\pi_t(X_i;\theta,\mu)\,a_{t,i},\\
\psi_i(\mu)=\tau\log\sum_t e^{a_{t,i}} .
\end{gather*}
Since $\log\pi_t(X_i;\theta,\mu)=a_{t,i}-\log\sum_s e^{a_{s,i}}$,
\[
\sum_t \pi_t(X_i;\theta,\mu)\log\pi_t(X_i;\theta,\mu)
=\sum_t \pi_t(X_i;\theta,\mu)\,a_{t,i}-\log\sum_s e^{a_{s,i}},
\]
and rearranging yields the identity
\begin{equation}
\label{eq:lse-entropy}
\psi_i(\mu)-\phi_i(\mu)=\tau\,H(\pi(X_i;\theta,\mu)).
\end{equation}
As $0\le H\le\log|\T|$ on the simplex, \eqref{eq:lse-entropy} gives $0\le\psi_i-\phi_i\le\tau\log|\T|$. Averaging over $i$ and adding the common term $\sum_t b_t\mu_t$ gives $\Dlo(\mu;M)\le\Dup(\mu;M)\le\Dlo(\mu;M)+\tau\log|\T|$, since $\tfrac1N\sum_i\phi_i(\mu)+\sum_t b_t\mu_t=\Dlo(\mu;M)$ and $\tfrac1N\sum_i\psi_i(\mu)+\sum_t b_t\mu_t=\Dup(\mu;M)$.
\end{proof}

\section{Feasibility in Expectation: Full Statement and Proof}
\label{app:feasibility}
The bilevel objective encourages feasibility through the inner price problem rather than an outer penalty. The following makes this precise for the convex surrogate $\Dup$: at the inner optimum the induced allocation respects every capacity on the sample, and the shadow price of a treatment is nonzero only when its budget is saturated.
\begin{proposition}[Feasibility in expectation, extended form of Proposition~\ref{prop:feas}]
\label{prop:feas-full}
Fix $M\in\R^{N\times|\T|}$ and $\tau>0$, and let $\mu^\star\in\argmin_{\mu\ge 0}\Dup(\mu;M)$. Write $\bar\pi_t(\mu)=\tfrac1N\sum_{i=1}^N\pi_t(X_i;\theta,\mu)$ for the average allocation to treatment $t$. Then
\begin{equation}
\label{eq:feas}
\bar\pi_t(\mu^\star)\le b_t\qquad\text{for every }t\in\T,
\end{equation}
and complementary slackness holds: $\mu^\star_t>0 \Rightarrow \bar\pi_t(\mu^\star)=b_t$. If moreover $\sum_{t\in\T}b_t>1$, which holds since $b_0=1$ and every scarce treatment has $b_t>0$, the minimiser $\mu^\star$ is unique. Finally, if $\{X_i\}$ are i.i.d.\ from $\Prob_X$ and $M_{i,t}=m_t(X_i;\theta)$, the population allocation satisfies $\E_X[\pi_t(X;\theta,\mu^\star)]\le b_t+\varepsilon_N$ with $\varepsilon_N=O_p(N^{-1/2})$ up to logarithmic factors.
\end{proposition}
\begin{proof}
$\Dup(\cdot;M)$ is convex and differentiable, with
\[
\frac{\partial \Dup}{\partial \mu_t}(\mu;M)=b_t-\frac1N\sum_{i=1}^N \pi_t(X_i;\theta,\mu)=b_t-\bar\pi_t(\mu),
\]
since $\partial_{\mu_t}\bigl[\tau\log\sum_s e^{(M_{i,s}-\mu_s)/\tau}\bigr]=-\pi_t(X_i;\theta,\mu)$. The problem $\min_{\mu\ge 0}\Dup$ has Lagrangian $\Dup(\mu;M)-\lambda^{\!\top}\mu$ with $\lambda\ge 0$, and KKT at $\mu^\star$ reads $\nabla_\mu \Dup(\mu^\star;M)=\lambda^\star\ge 0$, $\mu^\star_t\lambda^\star_t=0$. Stationarity is exactly $b_t-\bar\pi_t(\mu^\star)=\lambda^\star_t\ge 0$, i.e.\ \eqref{eq:feas}; complementary slackness gives $\mu^\star_t>0\Rightarrow\lambda^\star_t=0\Rightarrow\bar\pi_t(\mu^\star)=b_t$. Convexity with affine constraints makes these conditions necessary and sufficient.
\emph{Existence and uniqueness.} The Hessian of the log-sum-exp part of $\Dup$ is $\tfrac{1}{N\tau}\sum_i\bigl[\operatorname{diag}(\pi(X_i;\theta,\mu))-\pi(X_i;\theta,\mu)\pi(X_i;\theta,\mu)^{\!\top}\bigr]$, positive semidefinite with null space exactly $\operatorname{span}\{\mathbf 1\}$: for each $i$, $v^{\!\top}[\operatorname{diag}(\pi(X_i;\theta,\mu))-\pi(X_i;\theta,\mu)\pi(X_i;\theta,\mu)^{\!\top}]v=\operatorname{Var}_{\pi(X_i;\theta,\mu)}(v)$ vanishes only for constant $v$, since the weights have full support. Along $\mathbf 1$, $\Dup$ is affine with slope $\sum_t b_t-1>0$; along any other direction it is strictly convex. Along a ray $\mu+cv$ with $v\ge0$, $v\neq0$, the log-sum-exp part has slope tending to $-\min_t v_t$ while the linear part contributes $b^{\!\top}v>\min_t v_t$ when $\sum_t b_t>1$, so $\Dup$ is coercive on the nonnegative orthant and a minimiser exists. If two minimisers differed, convexity would make $\Dup$ affine on the segment joining them, forcing its direction into the null space $\operatorname{span}\{\mathbf 1\}$; but $\Dup$ has strictly positive slope along $\mathbf 1$, a contradiction. Hence $\mu^\star$ is unique.
\emph{Population version.} Each $\pi_t(X_i;\theta,\mu^\star)$ is a $[0,1]$-bounded function of $X_i$, so $\bar\pi_t(\mu^\star)$ is an empirical mean. For fixed $\mu$, Hoeffding gives an $O_p(N^{-1/2})$ deviation between $\bar\pi_t(\mu)$ and $\E_X[\pi_t(X;\theta,\mu)]$; since $\mu^\star$ is data-dependent, a supremum over $\mu$ in a compact price box, with the $\tau^{-1}$-Lipschitz dependence of the weights on $\mu$, turns this into a covering argument at the cost of a logarithmic factor. Combining with \eqref{eq:feas} yields the claim.
\end{proof}
Proposition~\ref{prop:feas-full} is the formal content of the claim that the dual-price layer builds capacity into the policy class: no capacity penalty appears in the outer objective, yet the inner optimum returns prices whose induced allocation is feasible in expectation, and the price of a treatment is the exact multiplier on its binding budget. This distinguishes the implicitly-priced policies from an unconstrained empirical-welfare maximiser \citep{kitagawa2018should,athey2021policy}, which optimises value with $\mu\equiv 0$ and is feasible only after post-hoc correction.
\begin{remark}[The nonconvex objective is approximately feasible]
\label{rem:F-feas}
For the softmax-weighted objective $\Dlo$ of \eqref{eq:bilevel}, averaging the entropy identity \eqref{eq:lse-entropy} over $i$ gives $\Dup-\Dlo=\tau\bar H$ with $\bar H(\mu)=\tfrac1N\sum_iH(\pi(X_i;\theta,\mu))$, hence $\nabla_\mu \Dlo=\nabla_\mu \Dup-\tau\nabla_\mu\bar H$. A stationary point of $\Dlo$ therefore satisfies, on its \emph{priced} coordinates,
\[
b_t-\bar\pi_t(\mu^\star_{\Dlo})=\tau\,\frac{\partial \bar H}{\partial\mu_t}(\mu^\star_{\Dlo}),
\]
a feasibility condition perturbed by an entropy-gradient correction whose sign is uncontrolled. Exact feasibility therefore need not hold for $\Dlo$; both objectives collapse to the hard dual as $\tau\downarrow0$; at the working temperature $\tau=0.03$, DPL-nc's average busiest-treatment usage in Table~\ref{tab:results} is at or below capacity on every dataset.
\end{remark}

\section{Training Loops}
\label{app:algorithms}
Algorithms~\ref{alg:alt} and~\ref{alg:f} state the alternating baseline and end-to-end training loops of Section~\ref{sec:method}. The convex pipeline \textbf{DPL-cvx} is Algorithm~\ref{alg:f} with the inner solve minimising $\Dup$ instead of $\Dlo$; convexity makes the backward's regularity assumptions hold automatically.
\begin{algorithm}[h]
\caption{Alternating-price baseline (\textbf{DPL-alt}).}
\label{alg:alt}
\begin{algorithmic}[1]
\Require Dataset $\{(X_i,T_i,Y_i,\widehat e_{T_i}(X_i))\}_{i=1}^N$; capacities $b$; temperature $\tau$; outer steps $S$; inner $\theta$-steps $K$.
\State Initialise $\theta$; set $\mu^{(0)}\!\gets 0$.
\For{$s = 1,\dots,S$}
  \State $\mu^{(s)} \gets \argmin_{\mu\ge 0}\,\Dlo(\mu;M_\theta)$ via L-BFGS-B, warm-started from $\mu^{(s-1)}$
  \For{$k = 1,\dots,K$}
    \State $M \gets (m_t(X_i;\theta))_{i,t}$
    \State $\pi_{i,t} \gets \softmax\!\bigl((M_i-\mu^{(s)})/\tau\bigr)_t$ \Comment{$\mu^{(s)}$ fixed}
    \State $\widehat V_{\mathrm{IPW}} \gets \tfrac{1}{N}\sum_{i=1}^N \pi_{i,T_i}\,Y_i / \widehat e_{T_i}(X_i)$
    \State $\theta \gets$ optimiser step on $-\widehat V_{\mathrm{IPW}}$ \Comment{no gradient through $\mu^{(s)}$}
  \EndFor
\EndFor
\State \Return $\theta$.
\end{algorithmic}
\end{algorithm}
\begin{algorithm}[h]
\caption{End-to-end pipeline (\textbf{DPL-nc}); replace $\Dlo$ by $\Dup$ for \textbf{DPL-cvx}.}
\label{alg:f}
\begin{algorithmic}[1]
\Require Dataset $\{(X_i,T_i,Y_i,\widehat e_{T_i}(X_i))\}_{i=1}^N$; capacities $b$; temperature $\tau$; outer steps $S$.
\State Initialise $\theta$; set warm-start prices $\mu^{(0)}\!\gets 0$.
\For{$s = 1,\dots,S$}
  \State $M \gets (m_t(X_i;\theta))_{i,t}\in\R^{N\times|\T|}$ \Comment{$M$ detached inside the inner solve}
  \State $\mu^\star \gets \argmin_{\mu\ge 0}\,\Dlo(\mu;M)$ via L-BFGS-B, warm-started from $\mu^{(s-1)}$
  \State $\pi_{i,t} \gets \softmax\!\bigl((M_i-\mu^\star)/\tau\bigr)_t$
  \State $\widehat V_{\mathrm{IPW}} \gets \tfrac{1}{N}\sum_{i=1}^N \pi_{i,T_i}\,Y_i / \widehat e_{T_i}(X_i)$
  \State Backward: propagate $\partial(-\widehat V_{\mathrm{IPW}})/\partial\theta$ through the policy and the implicit-differentiation layer of Appendix~\ref{app:ift}
  \State $\theta \gets$ optimiser step; $\mu^{(s)} \gets \mu^\star$
\EndFor
\State \Return $\theta$.
\end{algorithmic}
\end{algorithm}

\section{Implicit-Function Derivation of $\mathrm{d}\mu^\star/\mathrm{d}M$}
\label{app:ift}
Let $A(\mu)=\{t:\mu_t=0\}$ be the active set of the non-negativity constraint at $\mu$ and $\mathcal I(\mu)=\T\setminus A(\mu)$ the inactive set. The KKT conditions of the inner problem \eqref{eq:bilevel} are
\begin{equation}
\label{eq:kkt}
\nabla_\mu \Dlo(\mu;M)-\lambda =0,\quad \mu_t\,\lambda_t=0,\;\;\mu_t\ge 0,\;\;\lambda_t\ge 0,
\end{equation}
where $\lambda$ is the multiplier on $\mu\ge 0$. At a stationary point $\mu^\star$ with active set $A$, $\lambda_t^\star=\partial_{\mu_t}\Dlo(\mu^\star;M)$ on $A$ and $\lambda_t^\star=0$ on $\mathcal I$, so the inactive coordinates satisfy
\begin{equation}
\label{eq:foc-inactive}
\nabla_{\mu_{\mathcal I}} \Dlo(\mu^\star;M) = 0 .
\end{equation}
Differentiating \eqref{eq:foc-inactive} in $M$ (treating $A$ as locally constant) gives
\begin{equation}
\label{eq:ift-system}
H_{\mathcal I\mathcal I}\;\frac{\partial \mu^\star_{\mathcal I}}{\partial M}+\nabla^2_{\mu_{\mathcal I},M}\,\Dlo(\mu^\star;M)=0,
\end{equation}
where $H_{\mathcal I\mathcal I}=\nabla^2_{\mu_{\mathcal I}\mu_{\mathcal I}}\Dlo(\mu^\star;M)$. On the active set $\mu^\star_t\equiv 0$ locally, so $\partial \mu^\star_t/\partial M=0$ for $t\in A$.
\paragraph{Augmented KKT system used in the implementation.}
Rather than reducing to the inactive sub-block, the implementation differentiates the KKT conditions on $|\T|+|A|$ unknowns $(\mu^\star,\lambda_A)$. Stacking $\nabla_\mu \Dlo(\mu^\star;M)-\lambda=0$ over $\T$ with $\mu^\star_t=0$ for $t\in A$ and differentiating in $M$ gives the saddle-point system
\begin{equation}
\label{eq:ift-vjp}
\underbrace{\begin{bmatrix} H+\rho I_{|\T|} & -E_A^{\!\top} \\ E_A & 0 \end{bmatrix}}_{=:J}\,
\begin{bmatrix} \partial \mu^\star/\partial M \\ \partial \lambda_A/\partial M \end{bmatrix}
=-\begin{bmatrix} \nabla^2_{\mu,M} \Dlo(\mu^\star;M) \\ 0 \end{bmatrix},
\end{equation}
where $H=\nabla^2_\mu \Dlo(\mu^\star;M)$, $E_A\in\{0,1\}^{|A|\times|\T|}$ selects active coordinates, and $\rho=10^{-6}$ is a stabilising ridge. The bottom block enforces $(\partial \mu^\star/\partial M)_{A,\cdot}=0$, recovering \eqref{eq:ift-system} on the inactive coordinates. The implementation solves \eqref{eq:ift-vjp} by a single dense linear solve, takes the first $|\T|$ rows as $\partial \mu^\star/\partial M$, and forms the vector--Jacobian product $(\partial \mu^\star/\partial M)^{\!\top}g$ for the upstream gradient $g=\partial\mathcal L/\partial \mu^\star\in\R^{|\T|}$ with $\mathcal L=-\widehat V_{\mathrm{IPW}}$. The system is small ($|\T|+|A|\le 2|\T|$, $|\T|\le 10$ in our experiments), so the dense solve is cheap; the equivalent reduction $H_{\mathcal I\mathcal I}\,w_{\mathcal I}=g_{\mathcal I}$ on the inactive sub-block produces the same product.
\paragraph{Assumptions.}
For \eqref{eq:ift-system} to define $\mathrm{d}\mu^\star/\mathrm{d}M$ unambiguously we need (i)~a regular interior optimum on $\mathcal I$, so $H_{\mathcal I\mathcal I}$ is non-singular (the ridge makes this operational), and (ii)~strict complementarity, so the active set does not change instantaneously under a perturbation of $M$. These are the standard regularity conditions of differentiable optimisation layers \citep{amos2017optnet,agrawal2019differentiable}; under them the Implicit Function Theorem \citep{krantz2002implicit} gives local differentiability of $\mu^\star(M)$. Empirically, $H_{\mathcal I\mathcal I}$ was positive-definite at every iterate encountered. Re-classifying the active set before each backward makes the procedure robust to set changes between outer iterations.
\paragraph{Implementation notes.}
The forward is L-BFGS-B \citep{byrd1995limited} with bounds $[0,\infty)^{|\T|}$ and tolerances $10^{-10}$ (max $200$ iterations), warm-started from the previous outer iterate. The objective and gradient come from a small automatic-differentiation graph wrapped behind a numpy interface; $M$ and $b$ are detached inside the inner solve. The backward is a custom autograd function that (1)~identifies the active set $A=\{t:\mu^\star_t<10^{-7}\}$; (2)~rebuilds a differentiable graph for $\Dlo$ at $(\mu^\star,M)$ and materialises the $|\T|\times|\T|$ Hessian $H$ row by row; (3)~assembles $J$ of \eqref{eq:ift-vjp}; (4)~builds the right-hand side from the rows of $\nabla^2_{\mu,M}\Dlo(\mu^\star;M)$; (5)~solves $J\,U=-R$ in one dense solve and returns $\partial\mathcal L/\partial M=(\partial\mu^\star/\partial M)^{\!\top}g$. Capacities $b$ and temperature $\tau$ receive no gradients. The \textbf{DPL-cvx} pipeline uses a structurally identical implementation; convexity guarantees a positive-definite inner Hessian at the unique minimum.

\end{document}